\documentclass[a4paper]{article}
\usepackage{mystylefile}
\usepackage{natbib}
\usepackage{enumerate}
\usepackage{bm}
\usepackage{graphicx}
\bibpunct{(}{)}{;}{a}{,}{,}

\usepackage{epstopdf}
\usepackage[np,autolanguage]{numprint}
\pdfoutput=1

\newcommand{\B}{\mathbf{B}}

\newcommand{\M}{\mathbf{M}}
\renewcommand{\L}{\mathcal{L}}
\renewcommand{\C}{\mathbf{C}}
\newcommand{\Sig}{\mathbf{\Sigma}}
\newcommand{\dx}{d_{\bm{\mathrm{x}}}}
\newcommand{\ds}{d_{\bm{\mathrm{s}}}}

\newtheorem{gangtheorem}{Theorem}
\newtheorem{ganglemma}[gangtheorem]{Lemma}

\title{Non-Gaussian Component Analysis\\ with Log-Density
Gradient Estimation}
\author{Hiroaki Sasaki \\ hsasaki@is.naist.jp \\ Graduate School of
Information Science, \\ Nara Institute of Science \& Technology, Nara,
Japan \\\vspace{-1mm} \\
Gang Niu \\ 
gang@ms.k.u-tokyo.ac.jp \\
Graduate School of Frontier Sciences, \\ The University of Tokyo,
Chiba, Japan
\\\vspace{-1mm} \\
Masashi Sugiyama \\ 
sugi@k.u-tokyo.ac.jp \\
Graduate School of Frontier Sciences, \\ The University of Tokyo,
Chiba, Japan
}
\date{}
\begin{document}
\maketitle
\begin{abstract}
 Non-Gaussian component analysis (NGCA) is aimed at identifying a linear
 subspace such that the projected data follows a non-Gaussian
 distribution. In this paper, we propose a novel NGCA algorithm based on
 log-density gradient estimation. Unlike existing methods, the proposed
 NGCA algorithm identifies the linear subspace by using the eigenvalue
 decomposition without any iterative procedures, and thus is
 computationally reasonable. Furthermore, through theoretical analysis,
 we prove that the identified subspace converges to the true subspace at
 the optimal parametric rate. Finally, the practical performance of the
 proposed algorithm is demonstrated on both artificial and benchmark
 datasets.
\end{abstract}
 \section{Introduction}
 \label{sec:intro}
 A popular way to alleviate difficulties of handling high-dimensional
 data is to reduce the dimensionality of data. Real-world applications
 imply that a small number of non-Gaussian signal components in data
 often include ``interesting'' information, while the remaining Gaussian
 components are ``uninteresting''~\citep{blanchard2006search}. This is
 the fundamental motivation of non-Gaussian-based unsupervised dimension
 reduction methods.
 
 A well-known method is projection pursuit (PP), which estimates
 directions on which the projected data is as non-Gaussian as
 possible~\citep{friedman1974projection,huber1985projection}. In
 practice, PP algorithms maximize a \emph{single} index function
 measuring non-Gaussianity of the data projected on a
 direction. However, some index functions are suitable for measuring
 super-Gaussianity, while others are good at measuring
 sub-Gaussianity~\citep{hyvarinen2001independentBOOK}. Thus, PP
 algorithms might not work well when super- and sub-Gaussian signal
 components are mixtured in data.
 
 Non-Gaussian component analysis (NGCA)~\citep{blanchard2006search}
 copes with this problem. NGCA is a semi-parametric framework for
 unsupervised linear dimension reduction, and aimed at identifying a
 subspace such that the projected data follows a non-Gaussian
 distribution. Compared with independent component analysis
 (ICA)~\citep{comon1994independent,hyvarinen2001independentBOOK}, NGCA
 stands on a more general setting: There is no restriction about the
 number of Gaussian components and non-Gaussian signal components can be
 dependent of each other, while ICA makes a stronger assumption that
 \emph{at most} one Gaussian component is allowed and all the signal
 components are statistically independent of each other.
  
 To take into account both super- and sub-Gaussian components, the first
 practical NGCA algorithm called the \emph{multi-index projection
 pursuit} (MIPP) heuristically makes use of multiple index functions in
 PP~\citep{blanchard2006search}, but it seems to be unclear whether this
 heuristic works well in general. To improve the performance of MIPP,
 \emph{iterative metric adaptation for radial kernel functions} (IMAK)
 has been proposed~\citep{kawanabe2007new}. IMAK does not rely on index
 functions, but instead estimates alternative functions from
 data. However, IMAK involves an iterative optimization procedure, and
 its computational cost is expensive.
 
 In this paper, based on log-density gradient estimation, we propose a
 novel NGCA algorithm which we call the {least-squares NGCA}
 (LSNGCA). The rationale in LSNGCA is that as we show later, the target
 subspace contains the log-gradient for the data density subtracted by
 the log-gradient for a Gaussian density. Thus, the subspace can be
 identified using the eigenvalue decomposition. Unlike MIPP and IMAK,
 LSNGCA neither requires index functions nor any iterative procedures,
 and thus is computationally reasonable.
 
 A technical challenge in LSNGCA is to accurately estimate the gradient
 of the log-density for data. To overcome it, we employ a direct
 estimator called the \emph{least squares log-density gradients}
 (LSLDG)~\citep{cox1985penalty,SasakiHS14clustering}. LSLDG accurately
 and efficiently estimates log-density gradients in a closed form
 without going through density estimation. In addition, it includes an
 automatic parameter tuning method. In this paper, based on LSLDG, we
 theoretically prove that the subspace identified by LSNGCA converges to
 the true subspace at the optimal parametric rate, and finally
 demonstrate that LSNGCA reasonably works well on both artificial and
 benchmark datasets.
 
 This paper is organized as follows: In Section~\ref{sec:review}, after
 stating the problem of NGCA, we review MIPP and IMAK, and discuss their
 drawbacks. We propose LSNGCA, and then overview LSLDG in
 Section~\ref{sec:LSNGCA}. Section~\ref{sec:theory} performs a
 theoretical analysis of LSNGCA. The performance of LSNGCA on artificial
 datasets is illustrated in Sections~\ref{sec:illust}. Application to
 binary classification on benchmark datasets is given in
 Section~\ref{sec:application}. Section~\ref{sec:conc} concludes this
 paper.
 \section{Review of Existing Algorithms}
 \label{sec:review}
 In this section, we first describe the problem of NGCA, and then review
 existing NGCA algorithms.
 \subsection{Problem Setting}
 \label{ssec:problem}
 Suppose that a number of samples
 $\mathcal{X}=\{\vector{x}_i=(x_i^{(1)},x_i^{(2)},\dots,x_i^{(\dx)})^{\top}\}_{i=1}^n$
 are generated according to the following model:
 \begin{align}
  \vector{x}=\A\vector{s}+\vector{n}, \label{model}
 \end{align}
 where $\vector{s}=(s^{(1)},s^{(2)},\dots,s^{(\ds)})^{\top}$ denotes a
 random signal vector, $\A$ is a $\dx$-by-$\ds$ matrix, $\vector{n}$ is
 a Gaussian noise vector with the mean vector $\vector{0}$ and
 covariance matrix $\mathbf{C}$. Assume further that the dimensionality
 of $\vector{s}$ is lower than that of $\vector{x}$, namely $\ds<\dx$,
 and $\vector{s}$ and $\vector{n}$ are statistically independent of each
 other.
 
 Lemma~1 in~\citet{blanchard2006search} states that when data samples
 follow the generative model \eqref{model}, the probability density
 $p(\vector{x})$ can be described as a semi-parametric model:
 \begin{align}
  p(\vector{x})=f_{\bm{\mathrm{x}}}(\B^{\top}\vector{x})\phi_{\mathbf{C}}(\vector{x}),
  \label{density}
 \end{align}
 where $\B$ is a $\dx$-by-$\ds$ matrix, $f_{\bm{\mathrm{x}}}$ is a
 positive function and $\phi_{\mathbf{C}}$ denotes the Gaussian density
 with the mean $\vector{0}$ and covariance matrix $\mathbf{C}$.
 
 The decomposition in \eqref{density} is not unique because
 $f_{\bm{\mathrm{x}}}$, $\B$ and $\mathbf{C}$ are not identifiable from
 $p$. However, as shown in~\citet{theis2006uniqueness}, the following
 linear $\ds$-dimensional subspace is identifiable:
 \begin{align}
  \mathcal{L}=\text{Ker}(\B^{\top})^{\perp}=\text{Range}(\B).
  \label{space}
 \end{align}
 $\mathcal{L}$ is called the \emph{non-Gaussian index space}. Here, the
 problem is to identify $\L$ from $\mathcal{X}$. In this paper, we
 assume that $\ds$ is known.
 \subsection{Multi-Index Projection Pursuit}
 \label{ssec:MIPP}
 The first algorithm of NGCA called the \emph{multi-index projection
 pursuit} (MIPP) was proposed based on the following key
 result~\citep{blanchard2006search}:
 \begin{Prop}
  \label{key} Let $\vector{x}$ be a random variable whose density
  $p(\vector{x})$ has the semi-parametric form \eqref{density}, and
  suppose that $h(\vector{x})$ is a smooth real function on
  $\R{\dx}$. Denoting by $\I_{\dx}$ the $\dx$-by-$\dx$ identity matrix,
  assume further that $E\{\vector{x}\}=\vector{0}$ and
  $E\{\vector{x}\vector{x}^{\top}\}=\I_{\dx}$. Then, under mild
  regularity conditions on $h$, the following $\vector{\beta}(h)$
  belongs to the target space $\L$:
  \begin{align*}
   \vector{\beta}(h)=E\{\vector{x}h(\vector{x})
   -\nabla_{\vector{x}}h(\vector{x})\},
  \end{align*}
  where $\nabla_{\vector{x}}$ is the differential operator with respect
  to $\vector{x}$.
 \end{Prop}
 The condition that $E\{\vector{x}\vector{x}^{\top}\}=\I_{\dx}$ seems to
 be strong, but in practice it can be satisfied by whitening data. Based
 on Proposition~\ref{key}, after whitening data samples as
 $\vector{y}_i=\widehat{\Sig}^{-1/2}\vector{x}_i$ where
 $\widehat{\Sig}=\frac{1}{n}\sum_{i=1}^n\vector{x}_i\vector{x}_i^{\top}$,
 for a bunch of functions $\left\{h_k\right\}_{k=1}^K$, MIPP estimates
 $\vector{\beta}(h_k)=\vector{\beta}_k$ as
 \begin{align}
  \widehat{\vector{\beta}}_k
  =\frac{1}{n}\sum_{i=1}^n\vector{y}_ih_k(\vector{y}_i)
  -\nabla_{\vector{y}}h_k(\vector{y}_i). \label{empbeta}
 \end{align}
 Then, MIPP applies PCA to $\{\widehat{\vector{\beta}}_k\}_{k=1}^K$ and
 estimates $\L$ by pulling back the $\ds$-dimensional space spanned by
 the first $\ds$ principal directions into the original (non-whitened)
 space.
   
 Although the basic procedure of MIPP is simple, there are two
 implementation issues: normalization of $\widehat{\vector{\beta}}_k$
 and choice of functions $h_k$.  The normalization issue comes from the
 fact that since \eqref{empbeta} is a linear mapping,
 $\widehat{\vector{\beta}}_k$ with larger norm can be dominant in the
 PCA step, and they are not necessarily informative in practice.  To
 cope with this problem, \citet{blanchard2006search} proposed the
 following normalization scheme:
 \begin{align}
  \frac{\widehat{\vector{\beta}}_k}{\sqrt{\sum_{i=1}^n \|\vector{y}_i
  h_k(\vector{y}_i)-\nabla_{\vector{y}} h_k(\vector{y}_i)\|^2
  -\|\widehat{\vector{\beta}}_k\|^2}}. \label{normalization}
 \end{align}
 After normalization, since the squared norm of each vector is
 proportional to its signal-to-noise ratio, longer vectors are more
 informative.

 MIPP is supported by theoretical
 analysis~\citep[Theorem~3]{blanchard2006search}, but the practical
 performance strongly depends on the choice of $h$. To find an
 informative $h$, the form of $h$ was restricted as
 \begin{align*}
  h_{f,\vector{\omega}}(\vector{y})=r(\vector{\omega}^{\top}\vector{y}),
 \end{align*}
 where $\vector{\omega}\in\R{\dx}$ denotes a unit-norm vector, and $r$
 is a function. As a heuristic, the FastICA
 algorithm~\citep{hyvarinen1999fast} was employed to find a good
 $\vector{\omega}$. Although MIPP was numerically demonstrated to
 outperform PP algorithms, it is unclear whether these heuristic
 restriction and preprocessing work well in general.
 \subsection{Iterative Metric Adaptation for Radial Kernel Functions}
 \label{ssec:IMAK}
 To improve the performance of MIPP, the \emph{iterative metric
 adaptation for radial kernel functions} (IMAK) estimates $h$ by
 directly maximizing the informative normalization criterion, which is
 the squared norm of \eqref{normalization} used for normalization in
 MIPP~\citep{kawanabe2007new}. To estimate $h$, a linear-in-parameter
 model is used as
 \begin{align*}
  h_{\sigma^2,\M,\vector{\alpha}}(\vector{y})&=\sum_{i=1}^n \alpha_i
  \exp\left\{-\frac{1}{2\sigma^2}
  (\vector{y}-\vector{y}_i)^{\top}\M(\vector{y}-\vector{y}_i)\right\}\\
  &=\sum_{i=1}^n \alpha_i k_{\sigma^2,\M}(\vector{y},\vector{y}_i),
 \end{align*}
 where $\vector{\alpha}=(\alpha_1,\dots,\alpha_n)$ is a vector of
 parameters to be estimated, $\M$ is a positive semidefinite matrix and
 $\sigma$ is a fixed scale parameter. This model allows us to represent
 the squared norm of the informative criterion \eqref{normalization} as
 \begin{align}
  \frac{\|\widehat{\vector{\beta}}_k\|^2}{\sum_{i=1}^n \|\vector{y}_i
  h_k(\vector{y}_i)-\nabla_{\vector{y}} h_k(\vector{y}_i)\|^2
  -\|\widehat{\vector{\beta}}_k\|^2}=
  \frac{\vector{\alpha}^{\top}\mathbf{F}\vector{\alpha}}
  {\vector{\alpha}^{\top}\mathbf{G}\vector{\alpha}}.
  \label{IMAKcriterion}
 \end{align} 
 $\mathbf{F}$ and $\mathbf{G}$ in \eqref{IMAKcriterion} are given by
 \begin{align*}
  a  &\mathbf{F}=\frac{1}{n^2}\sum_{r=1}^{\dx}
  \left(\vector{e}_r^{\top}\mathbf{Y}\mathbf{K}
  -\vector{1}_n^{\top}\partial_{r}\mathbf{K}\right)^{\top}
  \left(\vector{e}_r^{\top}\mathbf{Y}\mathbf{K}
  -\vector{1}_n^{\top}\partial_{r}\mathbf{K}\right)\\
  &\mathbf{G}+\mathbf{F}\\ &=\frac{1}{n}\sum_{r=1}^{\dx}
  \left\{\text{diag}(\vector{e}_r^{\top}\mathbf{Y})\mathbf{K}
  -\partial_{r}\mathbf{K}\right\}^{\top}
  \left\{\text{diag}(\vector{e}_r^{\top}\mathbf{Y})\mathbf{K}
  -\partial_{r}\mathbf{K}\right\},
 \end{align*}
 where $\vector{e}_r$ denotes the $r$-th basis vector in $\R{\dx}$,
 $\mathbf{Y}$ is a $\dx$-by-$n$ matrix whose column vectors are
 $\vector{y}_i$, $\mathbf{K}$ is the Gram matrix whose $(i,j)$-th
 element is
 $[\mathbf{K}]_{ij}=k_{\sigma^2,\M}\left(\vector{y}_i,\vector{y}_j\right)$,
 $\partial_{r}$ denotes the partial derivative with respect to the
 $r$-th coordinate in $\vector{y}$, and
 \begin{align*}
  [\partial_{r}\mathbf{K}]_{ij}&=\frac{1}{\sigma^2}
  \left([\M\vector{y}_i]_r-[\M\vector{y}_j]_r\right)\\ &\times
  k^{\prime}_{\sigma^2,\M}\left(-\frac{1}{2\sigma^2}
  (\vector{y}_i-\vector{y}_j)^{\top}\M(\vector{y}_i-\vector{y}_j)\right).
 \end{align*}
 
 The maximizer of \eqref{IMAKcriterion} can be obtained by solving the
 following generalized eigenvalue problem:
 \begin{align*}
  \mathbf{F}\vector{\alpha}&=\eta\mathbf{G}\vector{\alpha},
 \end{align*} 
 where $\eta$ is the generalized eigenvalue. Once $\vector{\alpha}$ is
 estimated, $\vector{\beta}$ can be also estimated according to
 \eqref{empbeta}. Then, the metric $\M$ in $h$ is updated as
 \begin{align*}
  \M \propto \sum_{k}
  \widehat{\vector{\beta}}_k\widehat{\vector{\beta}}^{\top}_k,
 \end{align*}
 where $\M$ is scaled so that its trace equals to $\dx$. IMAK
 alternately and iteratively updates $\vector{\alpha}$ and
 $\vector{\beta}$. It was experimentally shown that IMAK improves the
 performance of MIPP. However, IMAK makes use of the above alternate and
 iterative procedure to estimate a number of functions
 $h_{\sigma^2,\M,\vector{\alpha}}$ with different parameter values for
 $\sigma$. Thus, IMAK is computationally costly.
 \section{Least-Squares Non-Gaussian Component Analysis (LSNGCA)}
 \label{sec:LSNGCA}
 In this section, we propose a novel algorithm for NGCA, which is based
 on the gradients of log-densities. Then, we overview an existing useful
 estimator for log-density gradients.
 \subsection{A Log-Density-Gradient-Based Algorithm for NGCA}
 \label{ssec:newalg}
 In contrast to MIPP and IMAK, our algorithm does not rely on
 Proposition~1, but is derived more directly from the semi-parametric
 model \eqref{density}. As stated before, the noise covariance matrix
 $\C$ in \eqref{density} cannot be identified in general. However, after
 whitening data, the semi-parametric model \eqref{density} is
 significantly simplified by following the proof of Lemma~3
 in~\citet{sugiyama2008approximating} as
 \begin{align}
  p(\vector{y})&= f_{\bm{\mathrm{y}}}(\B^{\prime\top}\vector{y})
  \phi_{\I_{\dx}}(\vector{y}), \label{wdensity}
 \end{align}
 where $\B^{\prime}$ is a $\dx$-by-$\ds$ matrix such that
 $\B^{\prime\top}\B^{\prime}=\I_{\ds}$,
 $\vector{y}=\Sig^{-1/2}\vector{x}$, $f_{\bm{\mathrm{y}}}$ is a positive
 function and $\Sig=E\{\vector{x}\vector{x}^{\top}\}$. Thus, under
 \eqref{wdensity}, the non-Gaussian index subspace can be represented as
 $\L=\text{Range}(\B)=\Sig^{-1/2}\text{Range}(\B^{\prime})$.
 
 To estimate $\text{Range}(\B^{\prime})$, we take a novel approach based
 on the gradients of log-densities. The reason of using the gradients
 comes from the following equation, which can be easily derived by
 computing the gradient of the both-hand sides of \eqref{wdensity} after
 taking the logarithm:
 \begin{align}
  \nabla_{\vector{y}}[\log p(\vector{y})-\log
  \phi_{\I_{\dx}}(\vector{y})]&= \B^{\prime} \nabla_{\vector{z}} \log
  f_{\bm{\mathrm{y}}}(\vector{z}=\B^{\prime\top}\vector{y}).
  \label{gradient}
 \end{align}
 Eq.\eqref{gradient} indicates that $\nabla_{\vector{y}}[\log
 p(\vector{y}) -\log\phi_{\I_{\dx}}(\vector{y})]=\nabla_{\vector{y}}\log
 p(\vector{y})+\vector{y}$ is contained in
 $\text{Range}(\B^{\prime})$. Thus, an orthonormal basis
 $\{\vector{e}_i\}_{i=1}^{\ds}$ in $\text{Range}(\B^{\prime})$ is
 estimated as the minimizer of the following PCA-like problem:
 \begin{align}
  E\{\|\vector{\nu}
  -\sum_{i=1}^{\ds}(\vector{\nu}^{\top}\vector{e}_i)\vector{e}_i\|^2\}
  &=E\{\|\vector{\nu}\|^2\} -\sum_{i=1}^{\ds}\vector{e}_i^{\top}
  E\{\vector{\nu}\vector{\nu}^{\top}\}\vector{e}_i,\label{LS}
 \end{align}
 where $\vector{\nu}=\nabla_{\vector{y}}\log p(\vector{y})+\vector{y}$,
 $\|\vector{e}_i\|=1$ and $\vector{e}_i^{\top}\vector{e}_j=0$ for $i\neq
 j$.  Eq.\eqref{LS} indicates that minimizing the left-hand side with
 respect to $\vector{e}_i$ is equivalent to maximizing the second term
 in the right-hand side. Thus, an orthonormal basis
 $\{\vector{e}_i\}_{i=1}^{\ds}$ can be estimated by applying the
 eigenvalue decomposition to
 $E\{\vector{\nu}\vector{\nu}^{\top}\}=E\{(\nabla_{\vector{y}}\log
 p(\vector{y})+\vector{y})(\nabla_{\vector{y}}\log
 p(\vector{y})+\vector{y})^{\top}\}$.
 
 The proposed LSNGCA algorithm is summarized in~Fig.\ref{alg}. Compared
 with MIPP and IMAK, LSNGCA estimates $\L$ without specifying or
 estimating $h$ and any iteration procedures. The key challenge in
 LSNGCA is to estimate log-density gradients $\nabla_{\vector{y}}\log
 p(\vector{y})$ in Step~2. To overcome this challenge, we employ a
 method called the \emph{least-squares log-density gradients}
 (LSLDG)~\citep{cox1985penalty,SasakiHS14clustering}, which directly
 estimates log-density gradients without going through density
 estimation. As overviewed below, with LSLDG, LSNGCA can compute all the
 solutions in a closed form, and thus would be a computationally
 efficient algorithm. 
  \begin{figure}[t]
   \centering \noindent\fbox{
   \begin{minipage}{\dimexpr\linewidth-2\fboxsep-2\fboxrule\relax}
   \textbf{Input:} Data samples, $\{\vector{x}_i\}_{i=1}^n$.
    \begin{enumerate}[Step~1]
    \item Whiten $\vector{x}_i$ after subtracting the empirical mean
	  values from them.
	  
    \item Estimate the gradient of the log-density for the whitened data
	  $\vector{y}_i=\widehat{\Sig}^{-1/2}\vector{x}_i$.
	  
    \item Using the estimated gradients
	  $\widehat{\vector{g}}(\vector{y}_i)$, compute
	  $\widehat{\mathbf{\Gamma}} =\frac{1}{n}\sum_{i=1}^n
	  \{\widehat{\vector{g}}(\vector{y}_i)+\vector{y}_i\}
	  \{\widehat{\vector{g}}(\vector{y}_i)+\vector{y}_i\}^{\top}$.

    \item Perform the eigenvalue decomposition to
	  $\widehat{\mathbf{\Gamma}}$, and let $\widehat{\mathcal{I}}$
	  be the space spanned by the $\ds$ directions corresponding to
	  the largest $\ds$ eigenvalues.
    \end{enumerate}
    \textbf{Output:}
    $\widehat{\mathcal{L}}=\widehat{\Sig}^{-1/2}\widehat{\mathcal{I}}$.
   \end{minipage}
   }
   \caption{\label{alg} The LSNGCA algorithm.}
  \end{figure}   
  \subsection{Least-Squares Log-Density Gradients (LSLDG)}
  \label{ssec:LSLDG}
  The fundamental idea of LSLDG is to directly fit a gradient model
  $g^{(j)}(\vector{x})$ to the true log-density gradient under the
  squared-loss:
  \begin{align*}
   &J(g^{(j)})\\ &=\int \left\{g^{(j)}(\vector{x})-\partial_{j}\log
   p(\vector{x})\right\}^2p(\vector{x})\d\vector{x}-C^{(j)}\\ &=\int
   \left\{g^{(j)}(\vector{x})\right\}^2p(\vector{x})\d\vector{x} -2\int
   g^{(j)}(\vector{x})\partial_{j} p(\vector{x})\d\vector{x}\\ &=\int
   \left\{g^{(j)}(\vector{x})\right\}^2p(\vector{x})\d\vector{x} +2\int
   \left\{\partial_{j} g^{(j)}(\vector{x})\right\}
   p(\vector{x})\d\vector{x},
  \end{align*}
  $C^{(j)}=\int \left\{\partial_{j}\log
  p(\vector{x})\right\}^2p(\vector{x})\d\vector{x}$,
  $\partial_{j}=\parder{x^{(j)}}$ and the last equality comes from the
  \emph{integration by parts} under a mild assumption that
  $\lim_{|x^{(j)}|\rightarrow\infty}g^{(j)}(\vector{x})p(\vector{x})=0$. Thus,
  $J(g^{(j)})$ is empirically approximated as
  \begin{align}
   \tilde{J}(g^{(j)})&=\frac{1}{n}\sum_{i=1}^n g^{(j)}(\vector{x}_i)^2
   +2\partial_{j} g^{(j)}(\vector{x}_i). \label{empJ}
  \end{align}

  To estimate log-density gradients, we use a linear-in-parameter model
  as 
  \begin{align*}
   g^{(j)}(\vector{x})&=\sum_{i=1}^b \theta_{ij} \psi_{ij}(\vector{x})
   =\vector{\theta}_{j}^{\top}\vector{\psi}_{j}(\vector{x}),
  \end{align*}
  where $\theta_{ij}$ is a parameter to be estimated,
  $\psi_{ij}(\vector{x})$ is a fixed basis function, and $b$ denotes the
  number of basis functions and is fixed to $b=\min(n,100)$ in this
  paper. As in~\citet{SasakiHS14clustering}, the derivatives of Gaussian
  functions centered at $\vector{c}_i$ are used for
  $\psi_{ij}(\vector{x})$:
  \begin{align*}
   \psi_{ij}(\vector{x})=\frac{[\vector{c}_i-\vector{x}]^{(j)}}{\sigma_j^2}
   \exp\left(-\frac{\|\vector{x}-\vector{c}_i\|^2}{2\sigma_j^2}\right),
  \end{align*}
  where $[\vector{x}]^{(j)}$ denotes the $j$-th element in $\vector{x}$,
  $\sigma_j$ is the width parameter, and the center point $\vector{c}_i$
  is randomly selected from data samples $\vector{x}_i$. After
  substituting the linear-in-parameter model and adding the $\ell_2$
  regularizer into \eqref{empJ}, the solution is computed analytically:
  \begin{align*}
   \widehat{\vector{\theta}}_{j}&=\argmin_{\vector{\theta}_{j}}
   \left[\vector{\theta}_{j}^{\top}
   \widehat{\mathbf{G}}_{j}\vector{\theta}_{j}
   +2\vector{\theta}_{j}^{\top}\widehat{\vector{h}}_{j}
   +\lambda_{j}\vector{\theta}_{j}^{\top}\vector{\theta}_{j}\right]\\
   &=-(\widehat{\mathbf{G}}_{j}+\lambda_{j}\I_{b})^{-1}\widehat{\vector{h}}_{j},
  \end{align*}
  where $\lambda_{j}$ denotes the regularization parameter,
  \begin{align*}
   \widehat{\mathbf{G}}_{j}=\frac{1}{n}\sum_{i=1}^n
   \vector{\psi}_{j}(\vector{x}_i)\vector{\psi}_{j}(\vector{x}_i)^{\top}
   \ \text{and}\ \widehat{\vector{h}}_{j}=\frac{1}{n}\sum_{i=1}^n
   \partial_{j}\vector{\psi}_{j}(\vector{x}_i).
  \end{align*}
  Finally, the estimator is obtained as
  \begin{align*}
   \widehat{g}^{(j)}(\vector{x})&=\widehat{\vector{\theta}}_{j}^{\top}
   \vector{\psi}_{j}(\vector{x}).
  \end{align*}
  
  As overviewed, LSLDG does not perform density estimation, but directly
  estimates log-density gradients. The advantages of LSLDG can be
  summarized as follows:
  \begin{itemize}
   \item The solutions are efficiently computed in a closed form.

   \item All the parameters, $\sigma_{j}$ and $\lambda_{j}$, can be
	 automatically determined by cross-validation.
	 
   \item Experimental results confirmed that LSLDG provides much more
	 accurate estimates for log-density gradients than an estimator
	 based on kernel density estimation especially for
	 higher-dimensional data~\citep{SasakiHS14clustering}.
  \end{itemize}
 \section{Theoretical Analysis}
 \label{sec:theory}
 We investigate the convergence rate of LSNGCA in a parametric
 setting. Recall that
 \begin{align*}
  \widehat{\mathbf{G}}_j=\frac{1}{n}\sum_{i=1}^n
  \vector{\psi}_j(\vector{x}_i)\vector{\psi}_j(\vector{x}_i)^{\top},
  \quad \widehat{\vector{h}}_j=\frac{1}{n}\sum_{i=1}^n
  \partial_j\vector{\psi}_j(\vector{x}_i),
 \end{align*}
 and denote their expectations by
 \begin{align*}
  \mathbf{G}_j^*=\mathbb{E}\left[\vector{\psi}_j(\vector{x})
  \vector{\psi}_j(\vector{x})^{\top}\right], \quad
  \vector{h}_j^*=\mathbb{E}\left[\partial_j\vector{\psi}_j(\vector{x})\right].
 \end{align*}
 Subsequently, let
 \begin{align*}
  \vector{\theta}_j^* &= \argmin\nolimits_{\vector{\theta}}
  \left\{\vector{\theta}^\top\mathbf{G}_j^*\vector{\theta}
  +2\vector{\theta}^{\top}\vector{h}_j^*
  +\lambda_j^*\vector{\theta}^\top\vector{\theta}\right\},\\
  g^{*(j)}(\vector{x}) &= \vector{\theta}_j^{*\top}
  \vector{\psi}_j(\vector{x}),\\
  \mathbf{\Gamma}^* &= \mathbb{E}\left[(\vector{g}^*(\vector{y})+\vector{y})
  (\vector{g}^*(\vector{y})+\vector{y})^{\top}\right],
 \end{align*}
 let $\mathcal{I}^*$ be the eigen-space of $\mathbf{\Gamma}^*$ with its
 largest $\ds$ eigenvalues, and
 $\mathcal{L}^*=\vector{\Sigma}^{-1/2}\mathcal{I}^*$ be the optimal
 estimate.
 
 \begin{gangtheorem}
  \label{thm:main} Given the estimated space $\widehat{\mathcal{L}}$
  based on a set of data samples of size $n$ and the optimal space
  $\mathcal{L}^*$, denote by
  $\widehat{\mathbf{E}}\in\mathbb{R}^{\dx\times\ds}$ the matrix form of
  an arbitrary orthonormal basis of $\widehat{\mathcal{L}}$ and by
  $\mathbf{E}^*\in\mathbb{R}^{\dx\times\ds}$ that of
  $\mathcal{L}^*$. Define the distance between spaces
  $\widehat{\mathcal{L}}$ and $\mathcal{L}^*$ as
  \begin{equation*}
   \mathcal{D}(\widehat{\mathcal{L}},\mathcal{L}^*)
    =\inf\nolimits_{\widehat{\mathbf{E}},\mathbf{E}^*}
    \|\widehat{\mathbf{E}}-\mathbf{E}^*\|_\mathrm{Fro},
  \end{equation*}
  where $\|\cdot\|_\mathrm{Fro}$ stands for the Frobenius norm. Then, as
  $n\to\infty$,
  \begin{equation*}
   \mathcal{D}(\widehat{\mathcal{L}},\mathcal{L}^*)
    =\mathcal{O}_p\left(n^{-1/2}\right),
  \end{equation*}
  provided that
  \begin{enumerate}
   \vspace{-1ex}%
    \itemsep1pt \parskip0pt \parsep0pt%6 
	 
   \item $\lambda_j$ for all $j$ converge in $\mathcal{O}(n^{-1/2})$ to
	 the non-zero limits, i.e.,
	 $\lim_{n\to\infty}n^{1/2}|\lambda_j-\lambda_j^*|<\infty$, and there
	 exists $\epsilon_\lambda>0$ such that
	 $\lambda_j^*\ge\epsilon_\lambda$;
	 
   \item $\psi_{ij}(\vector{x})$ for all $i$ and $j$ have well-chosen
	 centers and widths, such that the first $\ds$ eigenvalues of
	 $\mathbf{\Gamma}^*$ are neither $0$ nor $+\infty$.
  \end{enumerate}
 \end{gangtheorem}
 
 Theorem~\ref{thm:main} shows that LSNGCA is consistent, and its
 convergence rate is $\mathcal{O}_p(n^{-1/2})$ under mild
 conditions. The first is about the limits of $\ell_2$-regularizations,
 and it is easy to control. The second is also reasonable and easy to
 satisfy, as long as the centers are not located in regions with
 extremely low densities and the bandwidths are neither too large
 ($\widehat{\mathbf{\Gamma}}$ might be all-zero) nor too small
 ($\widehat{\mathbf{\Gamma}}$ might be unbounded).

 Our theorem is based on two powerful theories, one is of perturbed
 optimizations \citep{bonnans96,bonnans98}, and the other is of matrix
 approximation of integral operators
 \citep{koltchinskii98,koltchinskii00} that covers a theory of perturbed
 eigen-decompositions. According to the former, we can prove that
 $\widehat{\vector{\theta}}_j$ converges to $\vector{\theta}_j^*$ in
 $\mathcal{O}_p(n^{-1/2})$ and thus $\widehat{\mathbf{\Gamma}}$ to
 $\mathbf{\Gamma}^*$ in $\mathcal{O}_p(n^{-1/2})$. According to the
 latter, we can prove that $\widehat{\mathcal{I}}$ converges to
 $\mathcal{I}^*$ and therefore $\widehat{\mathcal{L}}$ to
 $\mathcal{L}^*$ in $\mathcal{O}_p\left(n^{-1/2}\right)$. The full proof
 can be found in Appendix~\ref{sec:proof}.
  \section{Illustration on Artificial Data}
  \label{sec:illust}
  \begin{figure}[t]
   \begin{center}
    \subfigure[Gaussian
    mixture]{\includegraphics[width=0.35\textwidth]{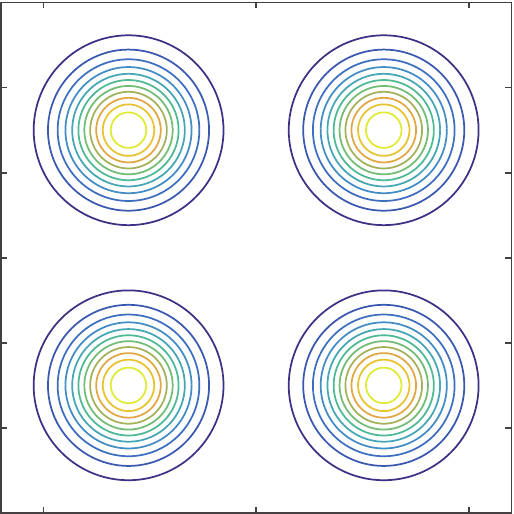}}
    \subfigure[Super-Gaussian]{\includegraphics[width=0.35\textwidth]{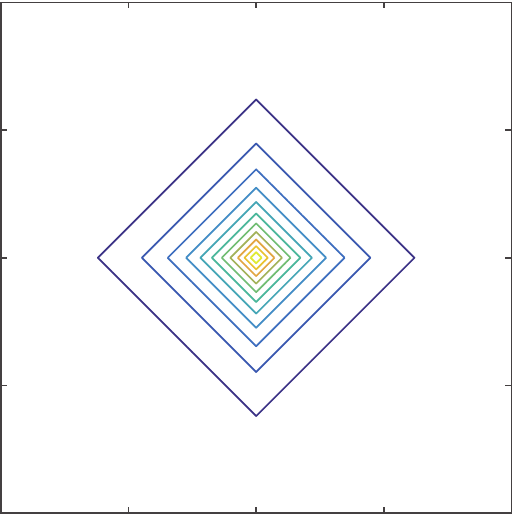}}
    \subfigure[Sub-Gaussian]{\includegraphics[width=0.35\textwidth]{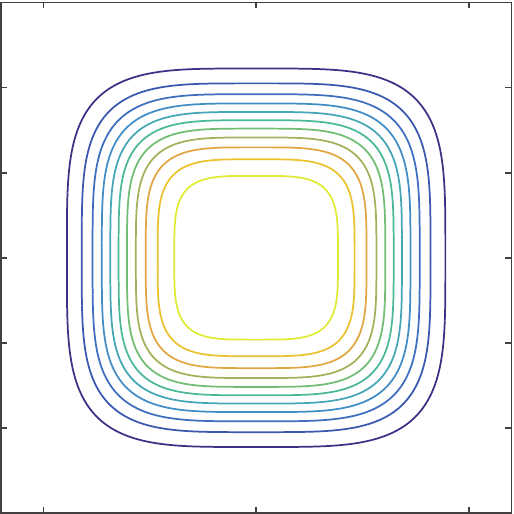}}
    \subfigure[Super- and
    sub-Gaussian]{\includegraphics[width=0.35\textwidth]{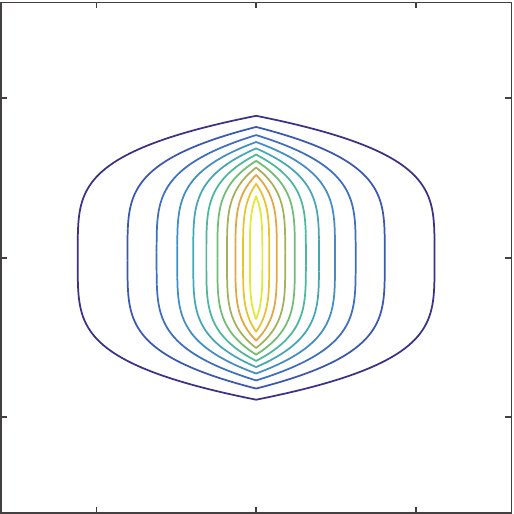}}
    \caption{\label{fig:dist} The two-dimensional distributions of four
    non-Gaussian densities.}
   \end{center}
  \end{figure}
  \begin{figure}[t]
   \begin{center}
    \subfigure[Gaussian
    mixture]{\includegraphics[width=0.45\textwidth]{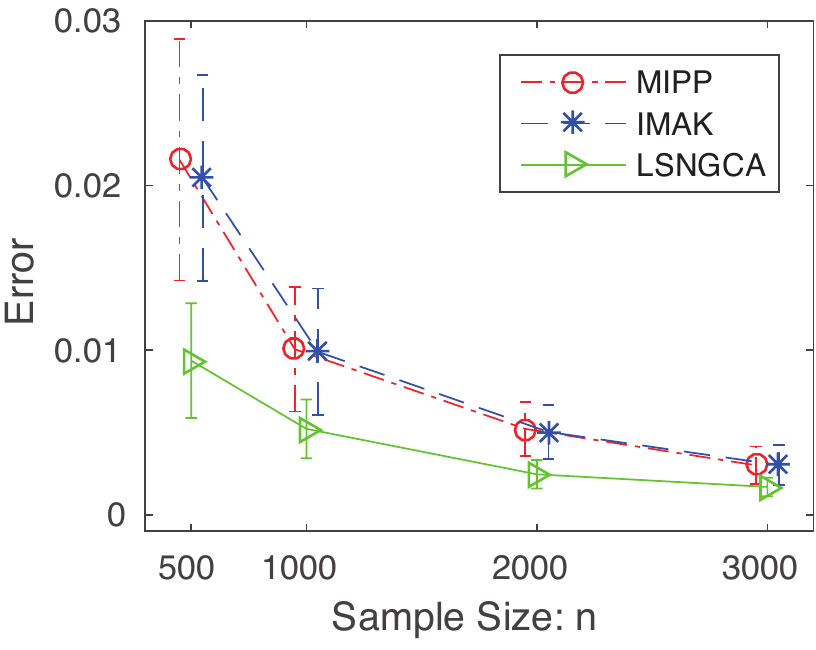}}
    \subfigure[Super-Gaussian]{\includegraphics[width=0.45\textwidth]{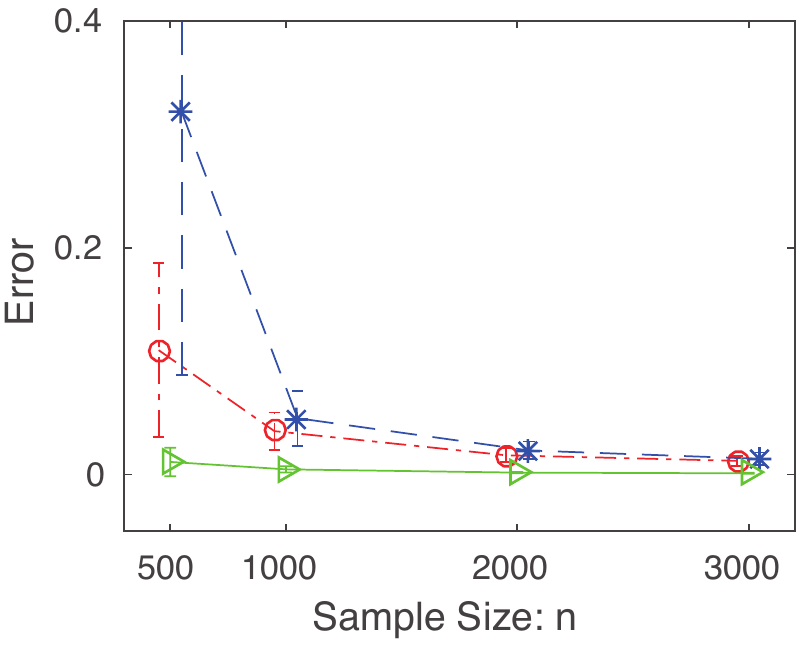}}
    \subfigure[Sub-Gaussian]{\includegraphics[width=0.45\textwidth]{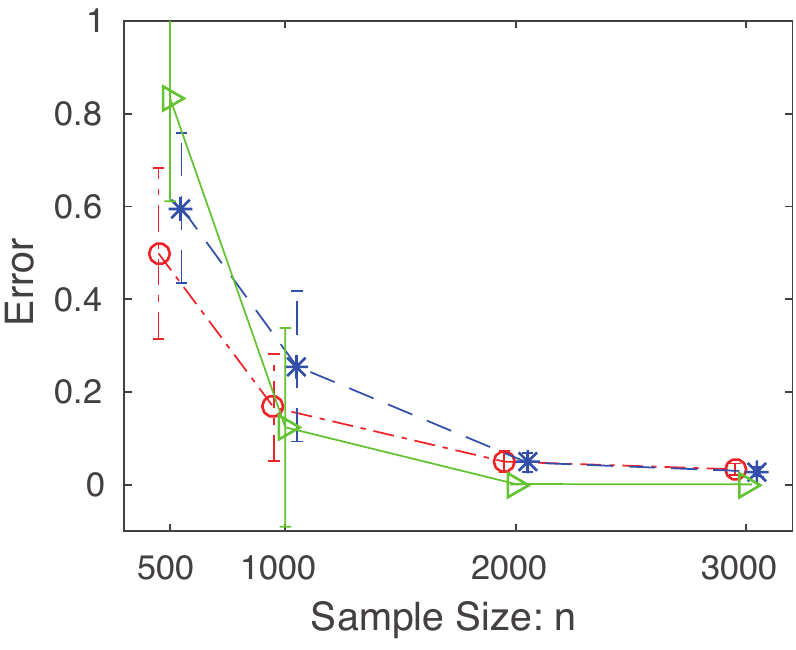}}
    \subfigure[Super- and
    sub-Gaussian]{\includegraphics[width=0.45\textwidth]{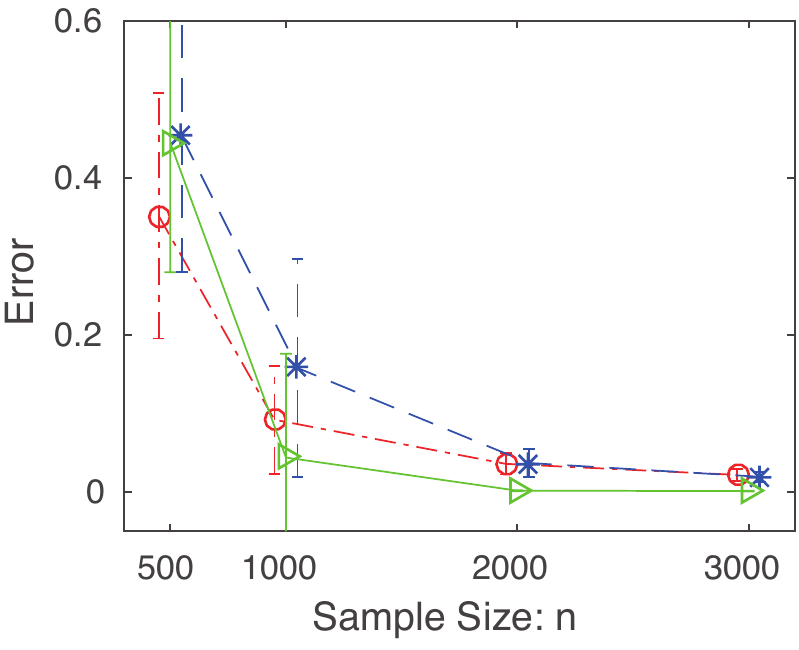}}
    \caption{\label{fig:error} The average errors over $50$ runs for
    four kinds of non-Gaussian signal components as the functions of
    samples size $n$. The error bars denote standard deviations.  The
    horizontal position of the markers for MIPP and IMAK was slightly
    modified to improve visibility of their error bars.}
  \end{center}
 \end{figure}
 \begin{figure}[t]
 \begin{center}
    \subfigure[Gaussian
    mixture]{\includegraphics[width=0.45\textwidth]{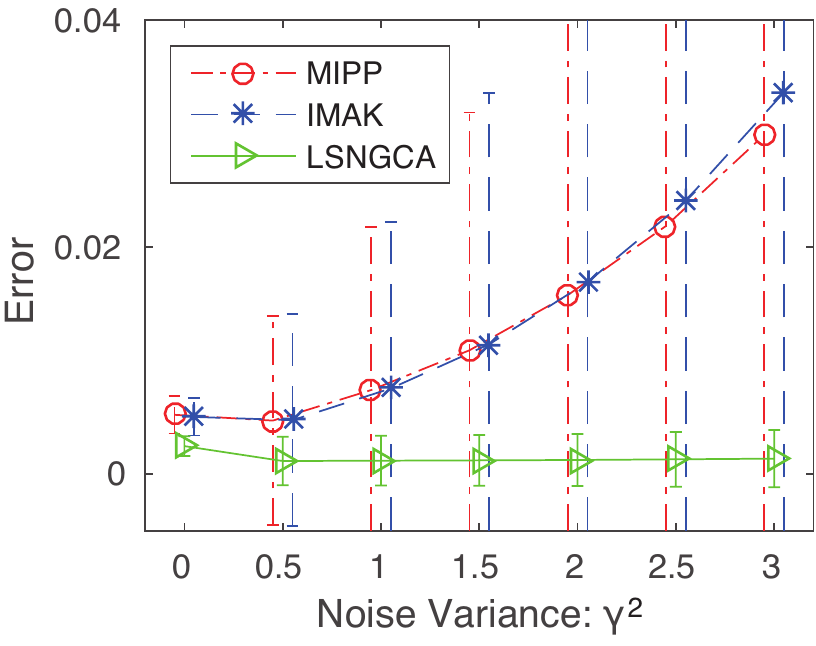}}
    \subfigure[Super-Gaussian]{\includegraphics[width=0.45\textwidth]{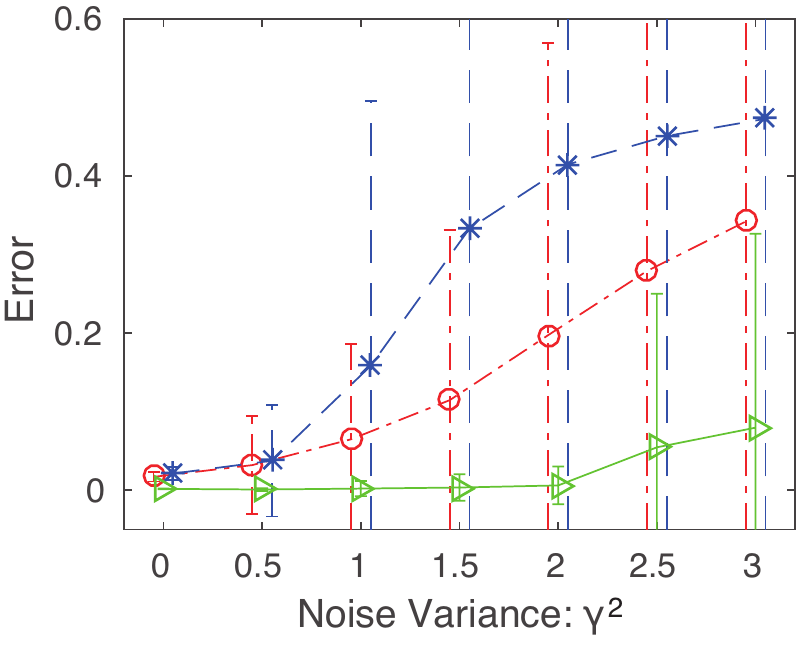}}
    \subfigure[Sub-Gaussian]{\includegraphics[width=0.45\textwidth]{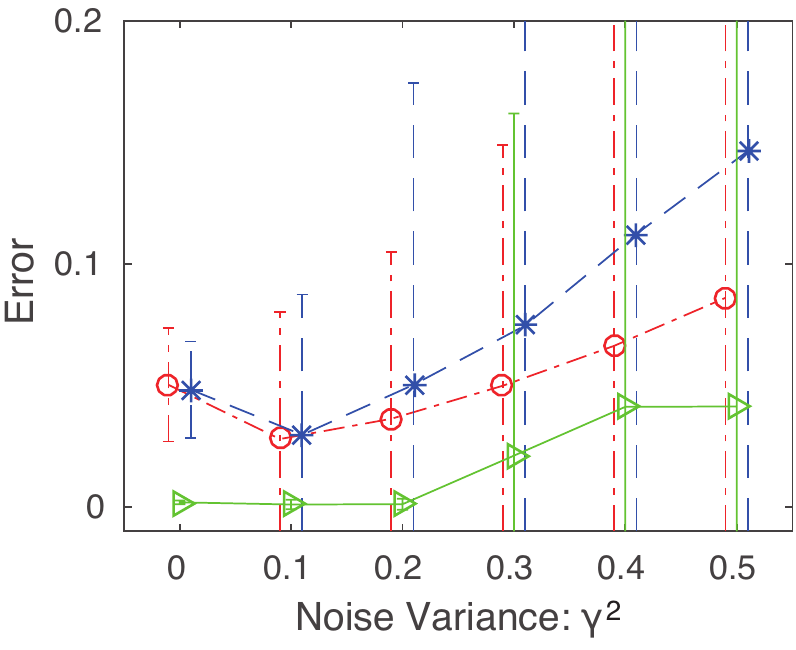}}
    \subfigure[Super- and
    sub-Gaussian]{\includegraphics[width=0.45\textwidth]{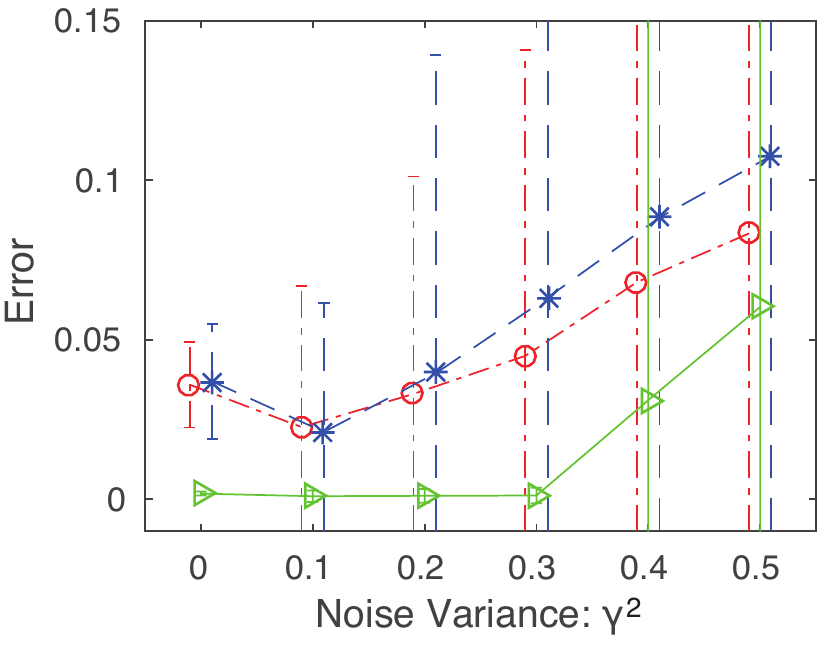}}
    \caption{\label{fig:sig} The average errors over $50$ runs for four
    kinds of non-Gaussian signal components as the functions of noise
    variances $\gamma^2$ when $n=2,000$. The horizontal position of the
    markers for MIPP and IMAK was slightly modified to improve
    visibility of their error bars.}
  \end{center}
 \end{figure}
  In this section, we experimentally illustrate how LSNGCA works on
  artificial data, and compare its performance with MIPP and IMAK.
 
 Non-Gaussian signal components $\vector{s}=(s_1,s_2)^{\top}$ were
 sampled from the following distributions:
 \begin{itemize}
  \item Gaussian mixture: $p(s_1,s_2)\propto\prod_{i=1}^2
	\exp\{-(s_i-3)^2/2\}+\exp\{-(s_i+3)^2/2\}$ (Fig.~\ref{fig:dist}(a)).
	
  \item Super-Gaussian: $p(s_1,s_2)\propto\prod_{i=1}^2
	\exp\left(-|s_i|/\alpha\right)$ where $\alpha$ is determined
	such that the variances of $s_1$ and $s_2$ are $3$
	(Fig.~\ref{fig:dist}(b)).
	
  \item Sub-Gaussian: $p(s_1,s_2)\propto\prod_{i=1}^2
	\exp(-s_i^4/\beta)$ where $\beta$ is determined such that the
	variances of $s_1$ and $s_2$ are $3$ (Fig.~\ref{fig:dist}(c)).
       
  \item Super- and sub-Gaussian: $p(s_1,s_2)=p(s_1)p(s_2)$ where
	$p(s_1)\propto\exp(-|s_1|/\alpha)$ and
	$p(s_2)\propto\exp(-s_2^4/\beta)$. $\alpha$ and $\beta$ is
	determined such that the variances of $s_1$ and $s_2$ are $3$
	(Fig.~\ref{fig:dist}(d)).
 \end{itemize}
 Then, data was generated according to
 $\vector{x}=(s_1,s_2,n_3,\dots,n_{10})$ where $n_i$ for $i=3,\dots,10$
 were sampled from the independent standard normal density. The error
 was measured by
 \begin{align*}
  \mathcal{E}(\widehat{\L},\L)= \frac{1}{\ds}\sum_{i=1}^{\ds}
  \|\widehat{\vector{e}}_i-\Pi_{\L}\widehat{\vector{e}}_i \|^2,
 \end{align*}
 where $\{\widehat{\vector{e}}_i\}_{i=1}^{\ds}$ is an orthonormal basis
 of $\widehat{\L}$, and $\Pi_{\L}$ denotes the orthogonal projection on
 $\L$. For model selection in LSLDG, a five-hold cross-validation was
 performed with respect to the hold-out error of \eqref{empJ} using the
 ten candidate values for $\sigma_j$ (or $\lambda_j$) from $10^{-1}$ (or
 $10^{-5}$) to $10$ at the regular interval in logarithmic scale .

 The results are presented in Fig.~\ref{fig:error}. For the Gaussian
 mixture and super-Gaussian cases, LSNGCA always works better than MIPP
 and IMAK even when the sample size is relatively small
 (Fig.~\ref{fig:error}(a) and (b)). On the other hand, when the signal
 components include sub-Gaussian components and the number of samples is
 insufficient, the performance of LSNGCA is not good
 (Fig.~\ref{fig:error}(c) and (d)). This presumably comes from the fact
 that estimating the gradients for logarithmic sub-Gaussian densities is
 more challenging than super-Gaussian densities. However, as long as the
 number of sample is sufficient, the performance of LSNGCA is comparable
 to or slightly better than other methods.
  
 Next, we investigate the performance of the three algorithms when the
 non-Gaussian signal components in data are contaminated by Gaussian
 noises such that $\vector{x}=(s_1+n_1,s_2+n_2,n_3,\dots,n_{10})$ where
 $n_1$ and $n_2$ are independently sampled from the Gaussian density
 with the mean $0$ and variance $\gamma^2$, while other $n_i$ for
 $i=3,\dots,10$ are sampled as in the last
 experiment. Fig.~\ref{fig:sig}(a) and (b) show that as $\gamma^2$
 increases, the estimation errors of LSNGCA for the Gaussian mixture or
 super-Gaussian distribution more mildly increases than MIPP and
 IMAK. When the data includes sub-Gaussian components, LSNGCA still
 works better than MIPP and IMAK for weak noise, but all methods are not
 robust to stronger noises.

 For computational costs, MIPP is the best method, while IMAK consumes
 much time (Fig.\ref{fig:ctime}). MIPP estimates a bunch of
 $\vector{\beta}_k$ by simply computing \eqref{empbeta}, and FastICA
 used in MIPP is an iterative method, but its convergence is
 fast. Therefore, MIPP is a quite efficient method. As reviewed in
 Section~\ref{ssec:IMAK}, because of the alternate and iterative
 procedure, IMAK is computationally demanding. LSNGCA is less efficient
 than MIPP, but its computational time is still reasonable.
  
 In short, LSNGCA is advantageous in terms of the sample size and noise
 tolerance especially when the non-Gaussian signal components follow
 multi-modal or super-Gaussian distributions. Furthermore, LSNGCA is not
 the most efficient algorithm, but its computational cost is reasonable.
  \begin{figure}[t]
   \begin{center}
    \includegraphics[width=0.45\textwidth]{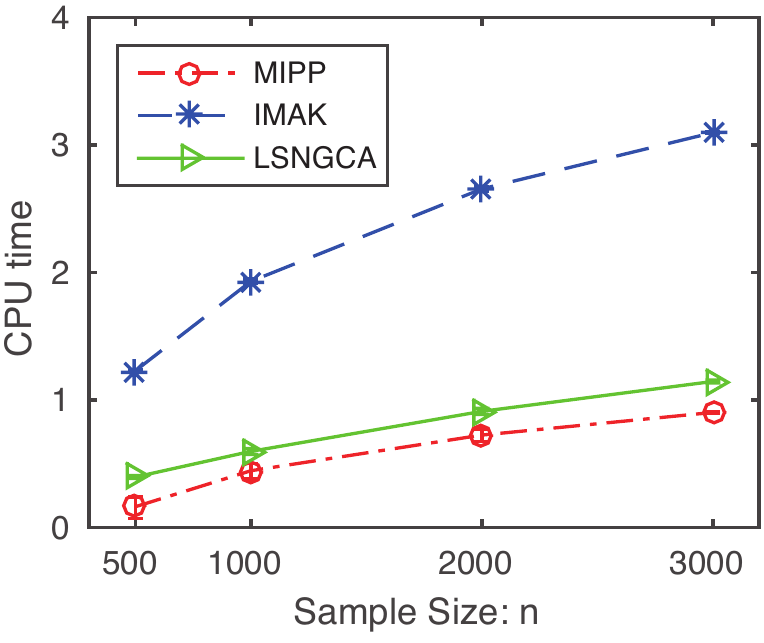}
    \caption{\label{fig:ctime} The average CPU time over $50$ runs for
    the Gaussian mixture as the functions of samples size $n$. The
    vertical axis is in logarithmic scale.}
   \end{center}
  \end{figure}
 \section{Application to Binary Classification on Benchmark Datasets}
 \label{sec:application}
 In this section, we apply LSNGCA to binary classification on benchmark
 datasets. For comparison, in addition to MIPP and IMAK, we employed PCA
 and locality preserving projections
 (LPP)~\citep{he2004locality}\footnote{\url{http://www.cad.zju.edu.cn/home/dengcai/Data/DimensionReduction.html}}.
 For LPP, the nearest-neighbor-based weight matrix were constructed
 using the heat kernel whose width parameter was fixed to $t_it_j$.
 $t_i$ is the Euclidean distance to the $k$-nearest neighbor sample of
 $\vector{x}_i$ and here we set $k=7$ as suggested
 by~\citet{zelnik2005self}.
 
 We used datasets for binary classification\footnote{The datasets,
 ``shuttle'' and ``vehicle'', originally include multiple classes. Here,
 to make datasets for binary classification, for ``shuttle'', we used
 only datasets corresponding to class 1 and 4, while for ``vehicle'', we
 assigned positive labels for class 1 and 3 datasets, and negative labels
 for the other datasets.} which are available at
 \url{https://www.csie.ntu.edu.tw/~cjlin/libsvmtools/datasets/}. For
 each dataset, we randomly selected $n$ samples for the training
 phase. The remaining samples were used for the test phase. For large
 datasets, we randomly chose $1,000$ samples for the training phase as
 well as for the test phase. As preprocessing, we separately subtracted
 the empirical means from the training and test samples. The projection
 matrix was estimated from the $n$ training samples by each
 method. Then, the support vector machine
 (SVM)~\citep{scholkopf2001learning} was trained using the
 dimension-reduced training data.\footnote{We employed a MATLAB software
 for SVM called \emph{LIBSVM}~\citep{CC01a}.}
 
 The averages and standard deviations for miss classification rates over
 $30$ runs are summarized in Table~\ref{tab:classification}. This table
 shows that LSNGCA overall compares favorably with other algorithms.
 \begin{table}[p]
  \begin{center}
   \caption{\label{tab:classification} The averages and standard
   deviations of misclassification rates for benchmark datasets over
   $30$ runs. The numbers in the parentheses are standard
   deviations. The best and comparable methods judged by the unpaired
   t-test at the significance level $1\%$ are described in boldface. The
   symbol ``-'' in the table means that IMAK unexpectedly stopped during
   the experiments because of a numerical problem.}
  \end{center}
\begin{center}
\begin{tabular}{c|c|c|c|c|c}
 \multicolumn{6}{c}{australian $(\dx, n)=(14, 200)$}\\ 
 & LSNGCA & MIPP & IMAK & PCA & LPP \\ \hline 
 $\ds=2$ & 20.20(5.09) & 21.02(6.66) & 33.43(4.99) & {\bf 17.37(1.30)} & {\bf 17.50(1.08)}\\
 $\ds=4$ & {\bf 16.23(2.60)} & {\bf 15.90(2.14)} & 32.53(6.06) & {\bf 14.92(1.17)} & {\bf 15.07(1.16)}\\
 $\ds=6$ & {\bf 15.41(2.32)} & {\bf 15.22(2.02)} & 30.71(5.71) & {\bf 14.16(1.16)} & {\bf 14.39(1.10)}
\end{tabular}
\end{center}
\begin{center}
\begin{tabular}{c|c|c|c|c|c}
\multicolumn{6}{c}{german.numer $(\dx, n)=(24, 200)$}\\ 
 & LSNGCA & MIPP & IMAK & PCA & LPP \\ \hline 
$\ds=2$ & {\bf 30.27(0.74)} & {\bf 30.35(0.77)} & - & {\bf 30.63(1.38)} & {\bf 30.82(1.52)}\\
$\ds=4$ & {\bf 30.29(0.62)} & {\bf 30.45(0.86)} & 31.12(1.22) & {\bf 29.90(1.68)} & {\bf 30.07(1.52)}\\
$\ds=6$ & 30.54(1.01) & 30.95(0.90) & 31.23(1.12) & {\bf 29.08(1.43)} & {\bf 29.46(1.09)}
\end{tabular}
\end{center}
\begin{center}
\begin{tabular}{c|c|c|c|c|c}
\multicolumn{6}{c}{liver-disorders $(\dx, n)=(6, 200)$}\\ 
 & LSNGCA & MIPP & IMAK & PCA & LPP \\ \hline 
$\ds=2$ & 39.31(3.62) & {\bf 32.62(3.72)} & {\bf 33.15(5.21)} & 42.14(2.71) & 42.00(2.96)\\
$\ds=4$ & {\bf 32.83(5.15)} & {\bf 32.02(3.67)} & 35.36(3.32) & 42.02(2.64) & 42.02(2.71)
\end{tabular}
\end{center}
\begin{center}
\begin{tabular}{c|c|c|c|c|c}
\multicolumn{6}{c}{SUSY $(\dx, n)=(18, 1000)$}\\ 
 & LSNGCA & MIPP & IMAK & PCA & LPP \\ \hline 
$\ds=2$ & {\bf 29.58(1.86)} & {\bf 29.42(1.70)} & 34.37(1.82) & {\bf 28.71(3.11)} & 35.26(1.87)\\
$\ds=4$ & {\bf 25.46(2.07)} & {\bf 25.91(1.70)} & 32.89(2.03) & 27.05(1.55) & 27.10(2.06)\\
$\ds=6$ & {\bf 23.32(1.73)} & 24.75(1.61) & 31.74(2.16) & 25.49(1.50) & 25.56(1.56)
\end{tabular}
\end{center}
\begin{center}
\begin{tabular}{c|c|c|c|c|c}
\multicolumn{6}{c}{shuttle $(\dx, n)=(9, 1000)$}\\ 
 & LSNGCA & MIPP & IMAK & PCA & LPP \\ \hline 
$\ds=2$ & {\bf 11.29(2.53)} & 14.39(3.34) & - & 16.01(2.20) & {\bf 11.41(3.53)}\\
$\ds=4$ & {\bf 6.04(3.24)} & 10.45(1.12) & 16.84(1.43) & 8.18(0.93) & 9.36(2.21)\\
$\ds=6$ & {\bf 3.03(1.73)} & 10.24(1.19) & 16.84(1.43) & 8.46(1.02) & 11.03(2.91)
\end{tabular}
\end{center}
\begin{center}
\begin{tabular}{c|c|c|c|c|c}
\multicolumn{6}{c}{vehicle $(\dx, n)=(18, 200)$}\\ 
 & LSNGCA & MIPP & IMAK & PCA & LPP \\ \hline 
$\ds=2$ & {\bf 41.23(4.26)} & 43.36(3.78) & 49.11(2.63) & {\bf 38.88(2.47)} & 46.97(2.44)\\
$\ds=4$ & {\bf 35.16(3.76)} & {\bf 34.26(4.13)} & 50.04(1.42) & 38.43(2.16) & 45.85(3.11)\\
$\ds=6$ & 30.72(3.95) & {\bf 26.60(2.24)} & 50.33(1.19) & 34.30(2.99) & 45.47(4.05)
\end{tabular}
\end{center}
\begin{center}
\begin{tabular}{c|c|c|c|c|c}
\multicolumn{6}{c}{svmguide3 $(\dx, n)=(21, 200)$}\\ 
 & LSNGCA & MIPP & IMAK & PCA & LPP \\ \hline 
$\ds=2$ & {\bf 22.58(1.55)} & {\bf 23.30(1.38)} & - & {\bf 23.22(1.12)} & 23.92(0.52)\\
$\ds=4$ & {\bf 22.32(1.59)} & {\bf 21.63(1.28)} & 23.93(0.52) & {\bf 21.74(0.92)} & 23.45(0.75)\\
$\ds=6$ & 22.20(1.54) & {\bf 21.29(0.96)} & 23.92(0.52) & 22.06(0.96) & 23.53(0.68)
\end{tabular}
\end{center}
 \end{table}
 \section{Conclusion}
 \label{sec:conc}
 In this paper, we proposed a novel algorithm for non-Gaussian component
 analysis (NGCA) called the \emph{least-squares NGCA} (LSNGCA). The
 subspace identification in LSNGCA is performed using the eigenvalue
 decomposition without any iterative procedures, and thus LSNGCA is
 computationally reasonable. Through theoretical analysis, we
 established the optimal convergence rate in a parametric setting for
 the subspace identification. The experimental results confirmed that
 LSNGCA performs better than existing algorithms especially for
 multi-modal or super-Gaussian signal components, and reasonably works
 on benchmark datasets.
\subsubsection*{Acknowledgements}
Hiroaki Sasaki did most of this work when he was working at the
university of Tokyo.
\bibliography{papers,new_bib}
\bibliographystyle{abbrvnat}
\appendix
\section{Proof of Theorem~1}
\label{sec:proof}%

Our proof can be divided into two parts as mentioned in
Section~\ref{sec:theory}.

\subsection{Part One: Convergence of LSLDG}

\subsubsection{Step 1.1}

First of all, we establish the growth condition of LSLDG (see
\emph{Definition~6.1} in~\citet{bonnans98} for the detailed definition
of the growth condition). Denote the expected and empirical objective
functions by
\begin{align*}
  J_j^*(\vector{\theta}) &=
  \vector{\theta}^{\top}\mathbf{G}_j^*\vector{\theta}
  +2\vector{\theta}^{\top}\vector{h}_j^*
  +\lambda_j^*\vector{\theta}^{\top}\vector{\theta},\\
  \widehat{J}_j(\vector{\theta}) &=
  \vector{\theta}^{\top}\widehat{\mathbf{G}}_j\vector{\theta}
  +2\vector{\theta}^{\top}\widehat{\vector{h}}_j
  +\lambda_j\vector{\theta}^{\top}\vector{\theta}.
\end{align*}
Then $\vector{\theta}_j^* = \argmin\nolimits_{\vector{\theta}}J_j^*(\vector{\theta})$ and $\widehat{\vector{\theta}}_j = \argmin\nolimits_{\vector{\theta}}\widehat{J}_j(\vector{\theta})$, and we have

\begin{ganglemma}
  \label{lem:grow}%
  The following second-order growth condition holds
  \begin{equation*}
    J_j^*(\vector{\theta}) \ge
    J_j^*(\vector{\theta}_j^*)+\epsilon_\lambda\|\vector{\theta}-\vector{\theta}_j^*\|_2^2.
  \end{equation*}
\end{ganglemma}
\begin{proof}
  $J_j^*(\vector{\theta})$ is strongly convex with parameter at least $2\lambda_j^*$, since $\mathbf{G}_j^*$ is symmetric and positive-definite. Hence,
  \begin{align*}
    J_j^*(\vector{\theta})
    &\ge J_j^*(\vector{\theta}_j^*)
    +(\nabla J_j^*(\vector{\theta}_j^*))^\top(\vector{\theta}-\vector{\theta}_j^*)
    +\lambda_j^*\|\vector{\theta}-\vector{\theta}_j^*\|_2^2\\
    &\ge J_j^*(\vector{\theta}_j^*)+\epsilon_\lambda\|\vector{\theta}-\vector{\theta}_j^*\|_2^2,
  \end{align*}
  where we used the optimality condition $\nabla J_j^*(\vector{\theta}_j^*)=\vector{0}$ and the first condition $\lambda_j^*\ge\epsilon_\lambda$ of the theorem.
\end{proof}

\subsubsection{Step 1.2}

Second, we study the stability (with respect to perturbation) of $J_j^*(\vector{\theta})$ at $\vector{\theta}_j^*$. Let
\begin{equation*}
  \vector{u}=\{\vector{u}_G\in\mathcal{S}_+^b,
  \vector{u}_h\in\mathbb{R}^b,
  u_\lambda\in\mathbb{R}\}
\end{equation*}
be a set of perturbation parameters, where $b$ is the number of centers in $\psi_{ij}(\vector{x})$ and $\mathcal{S}_+^b\subset\mathbb{R}^{b\times b}$ is the cone of $b$-by-$b$ symmetric positive semi-definite matrices. Define our perturbed objective function by
\begin{equation*}
  J_j(\vector{\theta},\vector{u}) =
  \vector{\theta}^{\top}(\mathbf{G}_j^*+\vector{u}_G)\vector{\theta}
  +2\vector{\theta}^{\top}(\vector{h}_j^*+\vector{u}_h)
  +(\lambda_j^*+u_\lambda)\vector{\theta}^{\top}\vector{\theta}.
\end{equation*}
It is clear that $J_j^*(\vector{\theta})=J_j(\vector{\theta},\vector{0})$, and then the stability of $J_j^*(\vector{\theta})$ at $\vector{\theta}_j^*$ can be characterized as follows.

\begin{ganglemma}
  \label{lem:lips}%
  The difference function $J_j(\cdot,\vector{u})-J_j^*(\cdot)$ is Lipschitz continuous modulus
  \begin{equation*}
    \omega(\vector{u}) = \mathcal{O}
    (\|\vector{u}_G\|_\mathrm{Fro} +\|\vector{u}_h\|_2 +|u_\lambda|)
  \end{equation*}
  on a sufficiently small neighborhood of $\vector{\theta}_j^*$.
\end{ganglemma}
\begin{proof}
  The difference function is
  \[ J_j(\vector{\theta},\vector{u})-J_j^*(\vector{\theta})
  = \vector{\theta}^{\top}\vector{u}_G\vector{\theta}
  +2\vector{\theta}^{\top}\vector{u}_h
  +u_\lambda\vector{\theta}^{\top}\vector{\theta}, \]
  with a partial gradient
  \[ \frac{\partial}{\partial\vector{\theta}}(J_j(\vector{\theta},\vector{u})-J_j^*(\vector{\theta}))
  = 2\vector{u}_G\vector{\theta} +2\vector{u}_h +2u_\lambda\vector{\theta}. \]
  Notice that due to the $\ell_2$-regularization in $J_j^*(\vector{\theta})$, $\exists M>0$ such that $\|\vector{\theta}_j^*\|_2\le M$. Now given a $\delta$-ball of $\vector{\theta}_j^*$, i.e., $B_\delta(\vector{\theta}_j^*) = \{\vector{\theta}\mid\|\vector{\theta}-\vector{\theta}_j^*\|_2\le\delta\}$, it is easy to see that $\forall\vector{\theta}\in B_\delta(\vector{\theta}_j^*)$,
  \[ \|\vector{\theta}\|_2
  \le \|\vector{\theta}-\vector{\theta}_j^*\|_2 +\|\vector{\theta}_j^*\|_2
  \le \delta+M, \]
  and consequently
  \[ \left\|\frac{\partial}{\partial\vector{\theta}}
  (J_j(\vector{\theta},\vector{u})-J_j^*(\vector{\theta}))\right\|_2 \le
  2(\delta+M)(\|\vector{u}_G\|_\mathrm{Fro}+|u_\lambda|)+2\|\vector{u}_h\|_2. \]
  This says that the gradient $\frac{\partial}{\partial\vector{\theta}}(J_j(\vector{\theta},\vector{u})-J_j^*(\vector{\theta}))$ has a bounded norm of order $\mathcal{O}(\|\vector{u}_G\|_\mathrm{Fro}+\|\vector{u}_h\|_2+|u_\lambda|)$, and proves that the difference function $J_j(\vector{\theta},\vector{u})-J_j^*(\vector{\theta})$ is Lipschitz continuous on the ball $B_\delta(\vector{\theta}_j^*)$, with a Lipschitz constant of the same order.
\end{proof}

\subsubsection{Step 1.3}

Intuitively, Lemma~\ref{lem:grow} guarantees that the unperturbed objective function $J_j^*(\vector{\theta})$ grows quickly when $\vector{\theta}$ leaves $\vector{\theta}_j^*$. Lemma~\ref{lem:lips} guarantees that the perturbed objective function $J_j(\vector{\theta},\vector{u})$ changes slowly for $\vector{\theta}$ around $\vector{\theta}_j^*$, where the slowness is with respect to the perturbation $\vector{u}$ it suffers. Based on Lemma~\ref{lem:grow}, Lemma~\ref{lem:lips}, and \emph{Proposition 6.1} in \citet{bonnans98},
\begin{equation*}
  \|\widehat{\vector{\theta}}_j-\vector{\theta}_j^*\|_2
  \le \frac{\omega(\vector{u})}{\epsilon_\lambda}
  = \mathcal{O}(\|\vector{u}_G\|_\mathrm{Fro}+\|\vector{u}_h\|_2+|u_\lambda|),
\end{equation*}
since $\widehat{\vector{\theta}}_j$ is the exact solution to $\widehat{J}_j(\vector{\theta})=J_j(\vector{\theta},\vector{u})$ given $\vector{u}_G=\widehat{\mathbf{G}}_j-\mathbf{G}_j^*$, $\vector{u}_h=\widehat{\mathbf{h}}_j-\mathbf{h}_j^*$, and $u_\lambda=\lambda_j-\lambda_j^*$. 
 
According to the \emph{central limit theorem} (CLT), $\|\vector{u}_G\|_\mathrm{Fro}=\mathcal{O}_p(n^{-1/2})$ and $\|\vector{u}_h\|_2=\mathcal{O}_p(n^{-1/2})$. Furthermore, we have already assumed that $|u_\lambda|=\mathcal{O}(n^{-1/2})$ in the first condition of the theorem. Hence, as $n\to\infty$,
\begin{equation}
  \label{eq:diff-theta}%
  \|\widehat{\vector{\theta}}_j-\vector{\theta}_j^*\|_2
  = \mathcal{O}_p\left(n^{-1/2}\right).
\end{equation}

\subsubsection{Step 1.4}

Considering the empirical estimate of the log-density gradient $\widehat{g}^{(j)}(\vector{x})$ and the optimal estimate of the log-density gradient $g^{*(j)}(\vector{x})$, their gap in terms of the infinity norm is bounded below:
\begin{align*}
  \|\widehat{g}^{(j)}-g^{*(j)}\|_\infty
  &= \sup\nolimits_{\vector{x}}|\widehat{g}^{(j)}(\vector{x})-g^{*(j)}(\vector{x})|\\
  &= \sup\nolimits_{\vector{x}}|(\widehat{\vector{\theta}}_j-\vector{\theta}_j^*)^{\top}
  \vector{\psi}_j(\vector{x})|\\
  &\le \|\widehat{\vector{\theta}}_j-\vector{\theta}_j^*\|_2
  \cdot \sup\nolimits_{\vector{x}}\|\vector{\psi}_j(\vector{x})\|_2,
\end{align*}
where the \emph{Cauchy-Schwarz inequality} is used. Recall that $\vector{c}_1,\ldots,\vector{c}_b$ are the centers, and for any $i$,
\begin{equation*}
  |\psi_{ij}(\vector{x})|
  = \frac{|[\vector{c}_i-\vector{x}]^{(j)}|}{\sigma_j^2}
  \exp\left(-\frac{\|\vector{x}-\vector{c}_i\|_2^2}{2\sigma_j^2}\right)
  \le \frac{|[\vector{c}_i-\vector{x}]^{(j)}|}{\sigma_j^2}
  \left(-\frac{([\vector{x}-\vector{c}_i]^{(j)})^2}{2\sigma_j^2}\right).
\end{equation*}
It is obvious that $|\psi_{ij}(\vector{x})|$ is bounded, since $\exp(-z^2)$ converges to zero much faster than $|z|$ diverges to infinity. Therefore, $\sup_{\vector{x}}\|\vector{\psi}_j(\vector{x})\|_2$ is a finite number, and we could know from Eq.~\eqref{eq:diff-theta} that
\begin{equation}
  \label{eq:diff-g}%
  \|\widehat{g}^{(j)}-g^{*(j)}\|_\infty
  \le \mathcal{O}(\|\widehat{\vector{\theta}}_j-\vector{\theta}_j^*\|_2)
  = \mathcal{O}_p\left(n^{-1/2}\right).
\end{equation}

\subsection{Part Two: Convergence of LSNGCA}

\subsubsection{Step 2.1}

To begin with, we focus on the convergence of $\widehat{\mathbf{\Gamma}}$. Given any $\vector{y}$, let $\widehat{\vector{z}}=\widehat{\vector{g}}(\vector{y})$ and $\vector{z}^*=\vector{g}^*(\vector{y})$, then
\begin{align*}
  (\widehat{\vector{z}}+\vector{y})(\widehat{\vector{z}}+\vector{y})^{\top}
  -(\vector{z}^*+\vector{y})(\vector{z}^*+\vector{y})^{\top}
  &= \widehat{\vector{z}}\widehat{\vector{z}}^{\top} -\vector{z}^*\vector{z}^{*\top}
  +(\widehat{\vector{z}}-\vector{z}^*)\vector{y}^{\top}
  +\vector{y}(\widehat{\vector{z}}-\vector{z}^*)^{\top}\\
  &= (\widehat{\vector{z}}-\vector{z}^*)\widehat{\vector{z}}^{\top}
  +\vector{z}^*(\widehat{\vector{z}}-\vector{z}^*)^{\top}
  +(\widehat{\vector{z}}-\vector{z}^*)\vector{y}^{\top}
  +\vector{y}(\widehat{\vector{z}}-\vector{z}^*)^{\top}.
\end{align*}
As a result, based on Eq.~\eqref{eq:diff-g},
\begin{align*}
  \|(\widehat{\vector{z}}+\vector{y})(\widehat{\vector{z}}+\vector{y})^{\top}
  -(\vector{z}^*+\vector{y})(\vector{z}^*+\vector{y})^{\top}\|_\mathrm{Fro}
  &\le \|(\widehat{\vector{z}}-\vector{z}^*)\widehat{\vector{z}}^{\top}\|_\mathrm{Fro}
  +\|\vector{z}^*(\widehat{\vector{z}}-\vector{z}^*)^{\top}\|_\mathrm{Fro}\\
  &\quad +\|(\widehat{\vector{z}}-\vector{z}^*)\vector{y}^{\top}\|_\mathrm{Fro}
  +\|\vector{y}(\widehat{\vector{z}}-\vector{z}^*)^{\top}\|_\mathrm{Fro}\\
  &\le (\|\widehat{\vector{z}}\|_2+\|\vector{z}^*\|_2+2\|\vector{y}\|)
  \cdot \|\widehat{\vector{z}}-\vector{z}^*\|_2\\
  &= \mathcal{O}(\|\widehat{\vector{z}}-\vector{z}^*\|_2)\\
  &= \mathcal{O}_p\left(n^{-1/2}\right).
\end{align*}
This has proved the point-wise convergence from $(\widehat{\vector{g}}(\vector{y})+\vector{y})(\widehat{\vector{g}}(\vector{y})+\vector{y})^{\top}$ to $(\vector{g}^*(\vector{y})+\vector{y})(\vector{g}^*(\vector{y})+\vector{y})^{\top}$.

Define an intermediate matrix based on $\vector{y}_1,\ldots,\vector{y}_n$ as
\begin{equation*}
  \widetilde{\mathbf{\Gamma}} =\frac{1}{n}\sum_{i=1}^n
  (\vector{g}^*(\vector{y}_i)+\vector{y}_i)
  (\vector{g}^*(\vector{y}_i)+\vector{y}_i)^{\top}.
\end{equation*}
Subsequently, $\widehat{\mathbf{\Gamma}}$ converges to $\widetilde{\mathbf{\Gamma}}$ in $\mathcal{O}_p(n^{-1/2})$ due to the point-wise convergence from $(\widehat{\vector{g}}(\vector{y})+\vector{y})(\widehat{\vector{g}}(\vector{y})+\vector{y})^{\top}$ to $(\vector{g}^*(\vector{y})+\vector{y})(\vector{g}^*(\vector{y})+\vector{y})^{\top}$ that was just proved, and $\widetilde{\mathbf{\Gamma}}$ converges to $\mathbf{\Gamma}^*$ in $\mathcal{O}_p(n^{-1/2})$ due to CLT. A combination of these two results gives us
\begin{equation}
  \label{eq:diff-gamma}%
  \|\widehat{\mathbf{\Gamma}}-\mathbf{\Gamma}^*\|_\mathrm{Fro}
  \le \|\widehat{\mathbf{\Gamma}}-\widetilde{\mathbf{\Gamma}}\|_\mathrm{Fro}
  +\|\widetilde{\mathbf{\Gamma}}-\mathbf{\Gamma}^*\|_\mathrm{Fro}
  = \mathcal{O}_p\left(n^{-1/2}\right).
\end{equation}

\subsubsection{Step 2.2}

Now let us consider the eigenvalues of $\mathbf{\Gamma}^*$. Let $\mu_1>\cdots>\mu_r>\mu_{r+1}$ be the first $r+1$ eigenvalues of $\mathbf{\Gamma}^*$ \emph{counted without multiplicity}, such that $\mu_r$ is the $\ds$-th largest eigenvalue of $\mathbf{\Gamma}^*$ if counted with multiplicity. Define the eigen-gap by
\begin{equation*}
  \epsilon_\mu=\min_{i=1,\ldots,r}\{\mu_i-\mu_{i+1}\}.
\end{equation*}
We have assumed that $\mu_1<+\infty$ and $\mu_r>0$ in the second condition of the theorem, and thus it must hold that $0<\epsilon_\mu<+\infty$. In the case that $\mathbf{\Gamma}^*$ has only one eigenvalue, we can simply assign $\epsilon_\mu=1$.

According to \emph{Lemma 5.2} of \citet{koltchinskii00} as well as the appendix of \citet{koltchinskii98}, we can derive the stability of the eigen-decomposition of $\mathbf{\Gamma}^*$ with respect to some perturbation $\vector{u}_\Sigma=\widehat{\mathbf{\Gamma}}-\mathbf{\Gamma}^*$. Whenever $\|\vector{u}_\Sigma\|_\mathrm{Fro}<\epsilon_\mu/4$:
\begin{itemize}
  \vspace{-1ex}%
  \itemsep1pt \parskip0pt \parsep0pt%
  \item The first $r+1$ eigenvalues $\mu_1'>\cdots>\mu_r'>\mu_{r+1}'$ of $\widehat{\mathbf{\Gamma}}=\mathbf{\Gamma}^*+\vector{u}_\Sigma$, counted without multiplicity, satisfy that $|\mu_i'-\mu_i|\le\|\vector{u}_\Sigma\|_\mathrm{Fro}$ for $1\le i\le r$, and $\mu_r-\mu_{r+1}'\ge\epsilon_\mu-\|\vector{u}_\Sigma\|_\mathrm{Fro}$;
  \item Denote by $\Pi_i(\mathbf{\Gamma}^*)$ the orthogonal projection onto the eigen-spaces of    $\mathbf{\Gamma}^*$ associated with $\mu_i$, and by $\Pi_i(\widehat{\mathbf{\Gamma}})$ that of $\widehat{\mathbf{\Gamma}}=\mathbf{\Gamma}^*+\vector{u}_\Sigma$ associated with $\mu_i'$, then for $1\le i\le r$,
  \[ \|\Pi_i(\widehat{\mathbf{\Gamma}})-\Pi_i(\mathbf{\Gamma}^*)\|_\mathrm{Fro}
  \le\frac{4}{\epsilon_\mu}\|\vector{u}_\Sigma\|_\mathrm{Fro}. \vspace{-1ex}\]
\end{itemize}
We have employed simplified notations above to avoid sophisticated names in operator theory. Intuitively, the first item guarantees that the eigenvalues of the perturbed matrix $\widehat{\mathbf{\Gamma}}$ are close to that of $\mathbf{\Gamma}^*$, and the second item guarantees that the eigen-spaces of $\widehat{\mathbf{\Gamma}}$ are also close to that of $\mathbf{\Gamma}^*$.

By noting that $\|\widehat{\mathbf{\Gamma}}-\mathbf{\Gamma}^*\|_\mathrm{Fro}$ was shown to have an order of $\mathcal{O}_p(n^{-1/2})$ in \eqref{eq:diff-gamma}, whereas the eigen-gap $\epsilon_\mu$ for fixed $\mathbf{\Gamma}^*$ is a constant value, we could obtain that as $n\to\infty$ for all $i$,
\begin{equation}
  \label{eq:diff-pi}%
  \|\Pi_i(\widehat{\mathbf{\Gamma}})-\Pi_i(\mathbf{\Gamma}^*)\|_\mathrm{Fro}
  = \mathcal{O}_p\left(n^{-1/2}\right).
\end{equation}

\subsubsection{Step 2.3}

Finally, we can bound $\mathcal{D}(\widehat{\mathcal{L}},\mathcal{L}^*)$. The eigenvalues of $\mathbf{\Gamma}^*$ and $\widehat{\mathbf{\Gamma}}$ were counted without multiplicity, and hence the bases of $\Pi_i(\widehat{\mathbf{\Gamma}})$ and $\Pi_i(\mathbf{\Gamma}^*)$ may not be unique. Nevertheless, let $\mathbf{E}_{\mathcal{I}^*}$ be the matrix form of a fixed orthonormal basis of $\mathcal{I}^*$, then there exists a sequence of matrices $\{\mathbf{E}_{\widehat{\mathcal{I}},1},\ldots,\mathbf{E}_{\widehat{\mathcal{I}},n},\ldots\}$ such that
\begin{itemize}
  \vspace{-1ex}%
  \itemsep1pt \parskip0pt \parsep0pt%
  \item $\mathbf{E}_{\widehat{\mathcal{I}},n}$ is the matrix form of a certain orthonormal basis of $\widehat{\mathcal{I}}$ based on a set of data samples of size $n$;
  \item The sequence converges to $\mathbf{E}_{\mathcal{I}^*}$ in $\mathcal{O}_p(n^{-1/2})$, i.e.,
  \begin{equation}
    \label{eq:diff-ei}%
    \|\mathbf{E}_{\widehat{\mathcal{I}},n}-\mathbf{E}_{\mathcal{I}^*}\|_\mathrm{Fro} = \mathcal{O}_p\left(n^{-1/2}\right), \vspace{-1ex}%
  \end{equation}
\end{itemize}
based on Eq.~\eqref{eq:diff-pi}. Denote by $\mathbf{E}_{\mathcal{L}^*}=\vector{\Sigma}^{-1/2}\mathbf{E}_{\mathcal{I}^*}$ and $\mathbf{E}_{\widehat{\mathcal{L}},n}=\widehat{\vector{\Sigma}}^{-1/2}\mathbf{E}_{\widehat{\mathcal{I}},n}$, and then
\begin{equation*}
  \mathbf{E}_{\widehat{\mathcal{L}},n}-\mathbf{E}_{\mathcal{L}^*}
  = \widehat{\vector{\Sigma}}^{-1/2}\mathbf{E}_{\widehat{\mathcal{I}},n}
  -\vector{\Sigma}^{-1/2}\mathbf{E}_{\mathcal{I}^*}
  = (\widehat{\vector{\Sigma}}^{-1/2}-\vector{\Sigma}^{-1/2})\mathbf{E}_{\widehat{\mathcal{I}},n}
  +\vector{\Sigma}^{-1/2}(\mathbf{E}_{\widehat{\mathcal{I}},n}-\mathbf{E}_{\mathcal{I}^*}).
\end{equation*}
Therefore, we can prove that
\begin{align*}
  \mathcal{D}(\widehat{\mathcal{L}},\mathcal{L}^*) &=
  \inf\nolimits_{\widehat{\mathbf{E}},\mathbf{E}^*}\|\widehat{\mathbf{E}}-\mathbf{E}^*\|_\mathrm{Fro}\\
  &\le
  \|\mathbf{E}_{\widehat{\mathcal{L}},n}-\mathbf{E}_{\mathcal{L}^*}\|_\mathrm{Fro}\\
  &\le \|\mathbf{E}_{\widehat{\mathcal{I}},n}\|_\mathrm{Fro} \cdot
  \|\widehat{\vector{\Sigma}}^{-1/2}-\vector{\Sigma}^{-1/2}\|_\mathrm{Fro}
  +\|\vector{\Sigma}^{-1/2}\|_\mathrm{Fro} \cdot
  \|\mathbf{E}_{\widehat{\mathcal{I}},n}-\mathbf{E}_{\mathcal{I}^*}\|_\mathrm{Fro}\\
  &=
  \mathcal{O}(\|\widehat{\vector{\Sigma}}^{-1/2}-\vector{\Sigma}^{-1/2}\|_\mathrm{Fro})
  +\mathcal{O}(\|\mathbf{E}_{\widehat{\mathcal{I}},n}-\mathbf{E}_{\mathcal{I}^*}\|_\mathrm{Fro})\\
  &= \mathcal{O}_p\left(n^{-1/2}\right),
\end{align*}
according to CLT and Eq.~\eqref{eq:diff-ei}. \qed

\end{document}